\documentclass{article}

\usepackage{times}
\usepackage{graphicx} 
\usepackage{subfigure} 

\usepackage{natbib}

\usepackage{algorithm}
\usepackage{algorithmic}

\usepackage{hyperref}

\usepackage[accepted]{icml2017}

\usepackage{epsfig}
\usepackage{graphicx}
\usepackage{amsmath}
\usepackage{amssymb}
\usepackage{amsfonts,latexsym,epsfig}
\usepackage{placeins}
\usepackage{enumitem}
\usepackage{url}
\usepackage{epstopdf}
\usepackage{authblk}
\usepackage{bbm}
\usepackage{tikz}

\newtheorem{theorem}{Theorem}[section]
\newtheorem{lemma}[theorem]{Lemma}
\newtheorem{definition}[theorem]{Definition}

\newtheorem{corollary}[theorem]{Corollary}

\newtheorem{assumptions}[theorem]{Assumptions}
\newenvironment{proof}{\par\noindent{\bf Proof:\ }}{\hfill$\Box$\\[2mm]}

\newif\ifpaper
\papertrue 

\def\RR{\mathbb{R}}

\def\<{\langle}
\def\>{\rangle}
\def\rank{\operatorname{\textit{rank}}}
\def\vec{\operatorname{\textit{vec}}}

\def\innerProd#1#2{\< #1, #2 \>}

\def\Set#1{\left\{ #1 \right\}}
\def\Bigbar#1{\mathrel{\left|\vphantom{#1}\right.}}
\def\Setbar#1#2{\Set{#1 \Bigbar{#1 #2} #2}}

\newcommand{\inner}[1]{\left\langle#1\right\rangle}

\newcommand{\norm}[1]{\left\|#1\right\|}

\def\bydef{\mathrel{\mathop:}=}

\def\det{\mathop{\rm det}\nolimits}

\def\ones{\mathop{\rm e}\nolimits}
\def\sign{\mathop{\rm sign}\limits}
\def\min{\mathop{\rm min}\nolimits}
\def\max{\mathop{\rm max}\nolimits}
\def\ones{\mathbf{1}}

\def\ie{\textit{i.e. }}
\def\eg{\textit{e.g. }}

\def\W{\mathcal{W}}
\def\B{\mathcal{B}}
\def\P{\mathcal{P}}

\icmltitlerunning{The Loss Surface of Deep and Wide Neural Networks}

\begin{document} 

\twocolumn[
\icmltitle{
    The Loss Surface of Deep and Wide Neural Networks
}

\icmlsetsymbol{equal}{*}

\begin{icmlauthorlist}
\icmlauthor{Quynh Nguyen}{to}
\icmlauthor{Matthias Hein}{to}
\end{icmlauthorlist}

\icmlaffiliation{to}{Department of Mathematics and Computer Science, Saarland University, Germany}

\icmlcorrespondingauthor{Quynh Nguyen}{quynh@cs.uni-saarland.de}

\icmlkeywords{local minima, global optimality, loss surface, deep neural networks, wide neural networks}

\vskip 0.3in
]



\printAffiliationsAndNotice{}  

\begin{abstract} 
   While the optimization problem behind deep neural networks is highly non-convex, it is frequently observed in practice
   that training deep networks seems possible without getting stuck in suboptimal points. 
   It has been argued that this is the case as all local minima are close to being globally optimal. 
   We show that this is (almost) true, in fact almost all local minima are globally optimal, 
   for a fully connected network with squared loss and analytic activation function 
   given that the number of hidden units of one layer of the network is larger than the number of training points 
   and the network structure from this layer on is pyramidal. 
\end{abstract} 

\section{Introduction}


The application of deep learning \cite{CunBenHin2015} has in recent years lead to a dramatic boost in performance
in many areas such as computer vision, speech recognition or natural language processing.
Despite this huge empirical success, the theoretical understanding of deep learning is still limited. 
In this paper we address the non-convex optimization 
problem of training a feedforward neural network. 
This problem turns out to be very difficult as there can be exponentially many distinct local minima \cite{Auer96,SafSha2016}.
It has been shown that the training of a network with a single neuron with a variety of activation functions 
turns out to be NP-hard \cite{Sim2002}. 

In practice local search techniques like stochastic gradient descent or variants are used for training deep neural networks. 
Surprisingly, it has been observed \cite{Dauphin16, Goodfellow15} that in the training of state-of-the-art 
feedforward neural networks with sparse connectivity like convolutional neural networks \cite{CunEtAl1990,KriSutHin2012} 
or fully connected ones one does not encounter problems with suboptimal local minima. 
However, as the authors admit themselves in \cite{Goodfellow15}, 
the reason for this might be that there is a connection between the fact that these networks have good performance 
and that they are easy to train. 

On the theoretical side there have been several interesting developments recently,
see \eg \cite{Brutzkus2017,LeeEtal2016,Poggio2017,Rister2017,Soudry17,Zhou2017}. 
For some class of networks one can show that one can train them globally optimal efficiently. 
However, it turns out that these approaches are either not practical \cite{Janzamin15,HaeVid2015,Soltanolkotabi2017} 
as they require e.g. knowledge about the data generating measure, 
or they modify the neural network structure and objective \cite{GauNgoHei2016}. 
One class of networks which are simpler to analyze are deep linear networks for which it has been shown that 
every local minimum is a global minimum \cite{Baldi88,Kawaguchi16}. 
While this is a highly non-trivial result as the optimization problem is non-convex,
deep linear networks are not interesting in practice as one efficiently just learns a linear function.
In order to characterize the loss surface for general networks, an interesting approach has been taken by \cite{Choro15}. 
By randomizing the nonlinear part of a feedforward network
with ReLU activation function and making some additional simplifying assumptions, 
they can relate it to a certain spin glass model which one can analyze. In this model the objective
of local minima is close to the global optimum and the number of bad local minima decreases quickly with the distance 
to the global optimum. This is a very interesting result but
is based on a number of unrealistic assumptions \cite{ChoroJLMR15}. 
It has recently been shown \cite{Kawaguchi16} that if some of these assumptions are dropped one basically recovers
the result of the linear case, but the model is still unrealistic. 

In this paper we analyze the case of overspecified neural networks, that is the network is larger than 
what is required to achieve minimum training error.
Under overspecification \cite{SafSha2016} have recently analyzed under which conditions 
it is possible to generate an initialization so that it is in principle possible to reach the global optimum 
with descent methods. 
However, they can only deal with one hidden layer networks and have to make strong assumptions on the data 
such as linear independence or cluster structure. In this paper overspecification means that there exists a very wide layer, 
where the number of hidden units is larger than the number of training points. For this case,
we can show that a large class of local minima is globally optimal. 
In fact, we will argue that almost every critical point is globally optimal. Our results generalize previous
work of \cite{Yu95}, who have analyzed a similar setting for one hidden layer networks, to networks of arbitrary depth. 
Moreover, it extends results of \cite{Gori92,Frasconi97} who have shown that for certain deep feedforward neural networks 
almost all local minima are globally optimal whenever the training data is linearly independent.  While it is clear that our
assumption on the number of hidden units is quite strong, 
there are several recent neural network structures which contain a quite wide hidden layer 
relative to the number of training points e.g. in \cite{LinEtAl2016} they have 50,000 training samples and the network has
one hidden layer with 10,000 hidden units and \cite{BaCar2014} have 1.1 million training samples and a layer with
400,000 hidden units. We refer to \cite{CirEtAl2010,NeyEtAl2015,VinEtal2010,Caruana01} for other examples where the number
of hidden units of one layer is on the order of the number of training samples.
We conjecture that for these kind of wide networks it still holds that almost all local minima are globally optimal. 
The reason is that one can expect linear separability of the training data in the wide layer. 
We provide supporting evidence for this conjecture by showing that basically every critical point
for which the training data is linearly separable in the wide layer is globally optimal. 
Moreover, we want to emphasize that all of our results hold for neural networks used in practice.
There are no simplifying assumptions as in previous work.

\section{Feedforward Neural Networks and Backpropagation}
We are mainly concerned with multi-class problems but our results also apply to multivariate regression problems.
Let $N$ be the number of training samples and 
denote by $X=[x_1,\ldots,x_N]^T\in\RR^{N\times d}, Y=[y_1,\ldots,y_N]^T\in\RR^{N\times m}$
the input resp. output matrix for the training data $(x_i,y_i)_{i=1}^N$, where $d$ is the input dimension
and $m$ the number of classes.
We consider fully-connected feedforward networks with $L$ layers, indexed from $0,1,2,\ldots,L,$
which correspond to the input layer, 1st hidden layer, etc, and output layer.
The network structure is determined by the weight matrices
$(W_k)_{k=1}^L\in \W:=\RR^{d\times n_1}\times\ldots\times \RR^{n_{k-1}\times n_k}\times\ldots\times \RR^{n_{L-1}\times m}$; 
where $n_k$ is the number of hidden units of layer $k$ (for consistency, we set $n_0=d,n_L=m$), 
and the bias vectors $(b_k)_{k=1}^L \in \B:=\RR^{n_1}\times \ldots \times\RR^{n_L}$.
We denote by $\mathcal{\P}=\W \times \B$ the space of all possible parameters of the network.
In this paper, $[a]$ denotes the set of integers $\Set{1,2,\ldots,a}$ 
and $[a,b]$ the set of integers from $a$ to $b$.
The activation function $\sigma:\RR\to\RR$ is assumed at least to be continuously differentiable, that is $\sigma \in C^1(\RR)$.
In this paper, we assume that all the functions are applied componentwise.
Let $f_k,g_k:\RR^d\to\RR^{n_k}$ be the mappings from the input space to the feature space at layer $k$,
which are defined as 
\[ f_0(x)=x,f_k(x) = \sigma(g_k(x)),g_k(x)=W_k^T f_{k-1}(x) + b_k\]
for every $k\in[L], x\in\RR^d.$ 
In the following, let $F_k=[f_k(x_1), f_k(x_2), \ldots, f_k(x_N)]^T\in\RR^{N\times n_k}$
and $G_k=[g_k(x_1), g_k(x_2), \ldots, g_k(x_N)]^T\in\RR^{N\times n_k}$ 
be the matrices that store the feature vectors of layer $k$ after and before applying the activation function.
One can easily check that
\begin{align*}
 F_1 &=\sigma(XW_1+\ones_N b_1^T) ,\\
 F_k &=\sigma(F_{k-1}W_k+\ones_N b_k^T), \; \textrm{ for } k\in[2,L].
\end{align*}
In this paper we analyze the behavior of the loss of the network without any form of regularization, 
that is the final objective $\Phi:\P \to \RR$ of the network is defined as
\begin{align}\label{eq:main1}
    \Phi\Big( (W_k,b_k)_{k=1}^L\Big) = \sum_{i=1}^N \sum_{j=1}^m l(f_{Lj}(x_i) - y_{ij}) 
\end{align}
where $l:\RR\to\RR$ is assumed to be a continuously differentiable loss function, that is $l \in C^1(\RR)$.
The prototype loss which we consider in this paper is the squared loss, $l(\alpha)=\alpha^2$, which is one of the standard
loss functions in the neural network literature. We assume throughout this paper that the minimum of \eqref{eq:main1} is attained.

The idea of backpropagation is the core of our theoretical analysis.
Lemma \ref{lem:grad} below shows well-known relations for feed-forward neural networks, which are
used throughout the paper. 
The derivative of the loss w.r.t. the value of unit $j$ at layer $k$ evaluated at a single training sample $x_i$ 
is denoted as $\delta_{kj}(x_i)=\frac{\partial\Phi}{\partial g_{kj}(x_i)}.$
We arrange these vectors for all training samples into a single matrix $\Delta_k$, defined as
\[ \Delta_k=[\delta_{k:}(x_1), \ldots, \delta_{k:}(x_N)]^T\in\RR^{N\times n_k}.\]
In the following we use the Hadamard product $\circ$, 
which for $A,B \in \RR^{m \times n}$ is defined as $A \circ B \in \RR^{m \times n}$ with
$(A \circ B)_{ij}=A_{ij}B_{ij}$.
\begin{lemma}\label{lem:grad}
    Let $\sigma,l\in C^1(\RR)$. Then it holds
    \begin{enumerate}
	\item $\Delta_k=
	    \begin{cases}
		l'(F_L-Y)\circ \sigma'(G_L), &k=L\\
		(\Delta_{k+1}W_{k+1}^T) \circ \sigma'(G_{k}), &k\in[L-1]
	    \end{cases}$
	\item $\nabla_{W_k}\Phi=
	    \begin{cases}
		X^T \Delta_1, &k=1\\
		F_{k-1}^T \Delta_k, & k\in[2,L]
	    \end{cases}$
	\item $\nabla_{b_k}\Phi=\Delta_k^T\ones_N$ $ \forall\,k\in[L]$
    \end{enumerate}
\end{lemma}
\ifpaper
\begin{proof}
    \begin{enumerate}
	\item By definition, it holds for every $i\in[N],j\in[n_L]$ that
	\begin{align*}
	    (\Delta_L)_{ij}&=\delta_{Lj}(x_i)\\
	    &=\frac{\partial\Phi}{\partial g_{Lj}(x_i)} \\
	    &=l'(f_{Lj}(x_i)-y_{ij}) \sigma'(g_{Lj}(x_i))\\
	    &=l'({(F_L)}_{ij}-Y_{ij}) \sigma'({(G_L)}_{ij}) 
	\end{align*}
	and hence, $\Delta_L=l'(F_L-Y)\circ\sigma'(G_L).$
	
	For every $k\in[L-1]$, the chain rule yields for every $i\in[N],j\in[n_k]$ that
	\begin{align*}
	    (\Delta_k)_{ij}&=\delta_{kj}(x_i)\\
	    &=\frac{\partial\Phi}{\partial g_{kj}(x_i)} \\
	    &=\sum_{l=1}^{n_{k+1}} \frac{\partial\Phi}{\partial g_{(k+1)l}(x_i)} \frac{\partial g_{(k+1)l}(x_i)}{\partial g_{kj}(x_i)}  \\
	    &=\sum_{l=1}^{n_{k+1}} \delta_{(k+1)l} (x_i) {(W_{k+1})}_{jl} \sigma'(g_{kj}(x_i)) \\
	    &=\sum_{l=1}^{n_{k+1}} {(\Delta_{(k+1)})}_{il} {(W_{k+1})^T}_{lj} \sigma'({(G_k)}_{ij}) 
	\end{align*}
	and hence $\Delta_k=(\Delta_{k+1}W_{k+1}^T) \circ \sigma'(G_{k}).$
	
	\item For every $r\in[d],s\in[n_1]$ it holds
	\begin{align*}
	    &\frac{\partial\Phi}{\partial (W_1)_{rs}} = \sum_{i=1}^N \frac{\partial\Phi}{\partial g_{1s}(x_i)}
	    \frac{\partial g_{1s}(x_i)}{\partial (W_1)_{rs}} \\
	    &=\sum_{i=1}^N \delta_{1s}(x_i) x_{ir} =\sum_{i=1}^N (X^T)_{ri} {(\Delta_1)}_{is} \\
	    &=\big(X^T\Delta_1\big)_{rs}
	\end{align*}
	and hence $\nabla_{W_1}\Phi=X^T\Delta_1.$
		
	For every $k\in[2,L],r\in[n_{k-1}],s\in[n_k],$ one obtains
	\begin{align*}
	    &\frac{\partial\Phi}{\partial (W_k)_{rs}} = \sum_{i=1}^N \frac{\partial\Phi}{\partial g_{ks}(x_i)}
	    \frac{\partial g_{ks}(x_i)}{\partial (W_k)_{rs}} \\
	    &=\sum_{i=1}^N \delta_{ks}(x_i) f_{(k-1)r}(x_i) =\sum_{i=1}^N (F_{k-1}^T)_{ri} {(\Delta_k)}_{is} \\
	    &=\big( F_{k-1}^T \Delta_k\big)_{rs}
	\end{align*}
	and hence $\nabla_{W_k}\Phi={F_{k-1}^T}\Delta_k.$
	
	\item 
	For every $k\in[1,L]$, $s\in[n_k]$ it holds
	\begin{align*}
	    &\frac{\partial\Phi}{\partial (b_k)_{s}} = \sum_{i=1}^N \frac{\partial\Phi}{\partial g_{ks}(x_i)}
	    \frac{\partial g_{ks}(x_i)}{\partial (b_k)_{s}} \\
	    &=\sum_{i=1}^N \delta_{ks}(x_i) = \big(\Delta_k^T\ones_N\big)_{s}
	\end{align*}
	and hence $\nabla_{b_k}\Phi=\Delta_k^T\ones_N.$
    \end{enumerate} 
\end{proof}
\fi
Note that Lemma \ref{lem:grad} does not apply to non-differentiable activation functions like the ReLU function, $\sigma_{\textrm{ReLU}}(x)=\max\{0,x\}$.
However, it is known that one can approximate this activation function arbitrarily well by 
a smooth function e.g. $\sigma_\alpha(x)=\frac{1}{\alpha}\log(1+e^{\alpha x})$ (a.k.a. softplus)
satisfies $\lim_{\alpha \rightarrow \infty} \sigma_\alpha(x)=\sigma_{\textrm{ReLU}}(x)$ for any $x \in \RR$.


\section{Main Result}\label{sec:main}
We first discuss some prior work and present then our main result together with extensive discussion. 
For improved readability we postpone the proof of the main result to the next section which contains several
intermediate results which are of independent interest.

\subsection{Previous Work}
Our work can be seen as a generalization of the work of \cite{Gori92, Yu95}. While \cite{Yu95} has shown that 
for a one-hidden layer network,
that if $n_1=N-1$, then every local minimum is a global minimum, the work of \cite{Gori92} considered also multi-layer networks.
For the convenience of the reader, we first restate Theorem 1 of \cite{Gori92} using our previously introduced notation.
The critical points of a continuously differentiable function $f:\RR^d \rightarrow \RR$ are the
points where the gradient vanishes, that is $\nabla f(x)=0$. Note that this is a necessary condition for a local minimum.
\begin{theorem}\cite{Gori92}\label{th:gori}
    Let $\Phi:\P\to\RR$ be defined as in \eqref{eq:main1} with least squares loss $l(a)=a^2.$
    Assume $\sigma:\RR\to[\underbar{d},\bar{d}]$ to be continuously differentiable with strictly positive derivative and
    $$\begin{array}{lr}
	\lim\limits_{a\to\infty} \frac{\sigma'(a)}{\bar{d}-\sigma(a)}>0, & \lim_{a\to\infty} \frac{-\sigma''(a)}{\bar{d}-\sigma(a)}>0 \\
	\lim_{a\to-\infty} \frac{\sigma'(a)}{\sigma(a)-\underbar{d}}>0, & \lim_{a\to-\infty} \frac{\sigma''(a)}{\sigma(a)-\underbar{d}}>0
    \end{array}$$
    Then every critical point $(W_l,b_l)_{l=1}^L$ of $\Phi$ which satisfies the conditions
    \begin{enumerate}
	\item $\rank(W_l)=n_l$ for all $l\in[2,L]$,
	\item $[X,\ones_N]^T\Delta_1=0$ implies $\Delta_1=0$
    \end{enumerate}
    is a global minimum.
\end{theorem}
While this result is already for general multi-layer networks, 
the condition ``$[X,\ones_N]^T\Delta_1=0$ implies $\Delta_1=0$'' is the main caveat.
It is already noted in \cite{Gori92}, that ``it is quite hard to understand its practical meaning'' 
as it requires prior knowledge of $\Delta_1$ at every critical point.
Note that this is almost impossible as $\Delta_1$ depends on all the weights of the network.
For a particular case, when the training samples (biases added) are linearly independent, \ie $\rank([X,\ones_N])=N$,
the condition holds automatically. 
This case is discussed in the following Theorem \ref{theo:independent_inputs}, where
we consider a more general class of loss and activation functions. 

\subsection{First Main Result and Discussion}
A function $f:\RR^d \rightarrow \RR$ is real analytic
if the corresponding multivariate Taylor series converges to $f(x)$ on an open subset of $\RR^d$ \cite{KraPar2002}.
All results in this section are proven under the following assumptions on the loss/activation function and training data. 
\begin{assumptions}\label{ass}
\begin{enumerate}
    \item There are no identical training samples, \ie $x_i \neq x_j$ for all $i\neq j$,
    \item $\sigma$ is analytic on $\RR$, strictly monotonically increasing and
           \begin{enumerate} 
            \item $\sigma$ is bounded  or
            \item there are positive $\rho_1,\rho_2,\rho_3,\rho_4$,
            s.t. $|\sigma(t)|\leq \rho_1 e^{\rho_2 t}$ for $t< 0$ and $|\sigma(t)|\leq \rho_3  t + \rho_4 $ for $t\geq 0$
          \end{enumerate}
    \item $l\in C^2(\RR)$ and if $l'(a)=0$ then $a$ is a global minimum
\end{enumerate}
\end{assumptions}
These conditions are not always necessary to prove some of the intermediate results presented below,
but we decided to provide the proof under the above strong assumptions for better readability.
For instance, all of our results also hold for strictly monotonically decreasing activation functions.
Note that the above conditions are not restrictive as many standard activation functions satisfy them.
\begin{lemma}
The sigmoid activation function $\sigma_1(t)=\frac{1}{1+e^{-t}}$, 
the tangent hyperbolic $\sigma_2(t)=\mathrm{tanh}(t)$ 
and the softplus function $\sigma_3(t)=\frac{1}{\alpha}\log(1+e^{\alpha t})$ for $\alpha>0$ 
satisfy Assumption \ref{ass}.
\end{lemma}
\ifpaper
\begin{proof}
Note that $\sigma_2(t)=\frac{2}{1+e^{-2t}}-1$. 
Moreover, it is well known that $\phi(t)=\frac{1}{1+t}$ is real-analytic on 
$\RR_+=\{t \in \RR \,|\,t\geq 0\}$. 
The exponential function is analytic with values in $(0,\infty)$.
As composition of real-analytic function is real-analytic (see Prop 1.4.2 in \cite{KraPar2002}), 
we get that $\sigma_1$ and $\sigma_2$ are real-analytic.
Similarly, since $\log(1+t)$ is real-analytic on $(-1,\infty)$ and the composition with the exponential function is real-analytic,
we get that $\sigma_3$ is a real-analytic function.

Finally, we note that $\sigma_1$,$\sigma_2,\sigma_3$ are strictly monotonically increasing. 
Since $\sigma_1$,$\sigma_2$ are bounded, they both satisfy Assumption \ref{ass}.
For $\sigma_3$, we note that $1+e^{\alpha t} \leq 2e^{\alpha t}$ for $t\geq 0$, and thus it holds for every $t\geq 0$ that
\begin{align*} 
    0 \leq \sigma_3(t) 
    &= \frac{1}{\alpha}\log(1+e^{\alpha t}) \\
    &\leq \frac{1}{\alpha}\log(2 e^{\alpha t}) \\
    &= \frac{\log(2)}{\alpha} + t,
\end{align*}
and with $\log(1+x)\leq x$ for $x>-1$ it holds $\log(1+e^{\alpha t}) \leq e^{\alpha t}$ for every $t\in\RR$. 
In particular \[ 0 \leq \sigma_3(t) \leq \frac{ e^{\alpha t}}{\alpha} \quad \forall t<0\]
which implies that $\sigma_3$ satisfies Assumption \ref{ass} 
for $\rho_1=1/\alpha,\rho_2=\alpha,\rho_3=1,\rho_4=\log(2)/\alpha.$
\end{proof}
\fi
The conditions on $l$ are satisfied for any twice continuously differentiable convex loss function.
A typical example is the squared loss $l(a)=a^2$ 
or the Pseudo-Huber loss \cite{Hartley2004} given as $l_\delta(a)=2\delta^2(\sqrt{1+a^2/\delta^2}-1)$ 
which approximates $a^2$ for small $a$ and is linear with slope $2\delta$ for large $a.$
But also non-convex loss functions satisfy this requirement, for instance:
\begin{enumerate}
    \item Blake-Zisserman: $l(a)=-\log(\exp(-a^2)+\delta)$ for $\delta>0.$
    For small $a$, this curve approximates $a^2$, whereas for large $a$ the asymptotic value is $-\log(\delta).$
    \item Corrupted-Gaussian: $$l(a)=-\log\big(\alpha\exp(-a^2)+(1-\alpha)\exp(-a^2/w^2)/w\big)$$
     for $\alpha\in[0,1], w>0.$
    This function computes the negative log-likehood of a gaussian mixture model.
    \item Cauchy: $l(a)=\delta^2\log(1+a^2/\delta^2)$  for $\delta\neq 0.$
    This curve approximates $a^2$ for small $a$ and the value of $\delta$ determines for
    what range of $a$ this approximation is close.
\end{enumerate}
We refer to \cite{Hartley2004} (p.617-p.619) for more examples and discussion on robust loss functions.
\begin{figure*}[ht]
\begin{center}
    \subfigure[Blake-Zisserman ($\delta=0.1$)]{\includegraphics[width=0.23\linewidth]{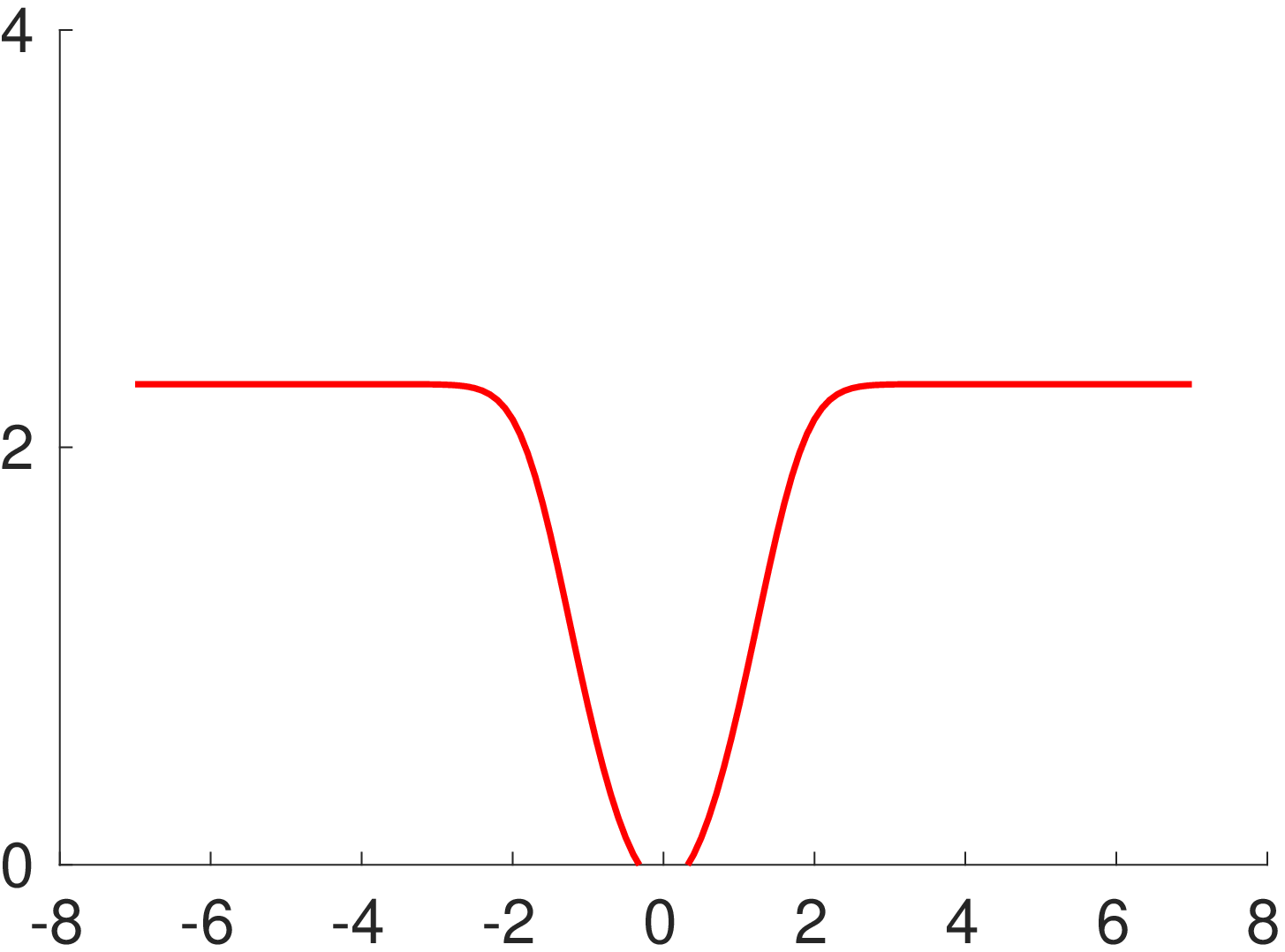}}
    \subfigure[Corrupted-Gaussian ($\alpha=0.5,w=6$)]{\includegraphics[width=0.23\linewidth]{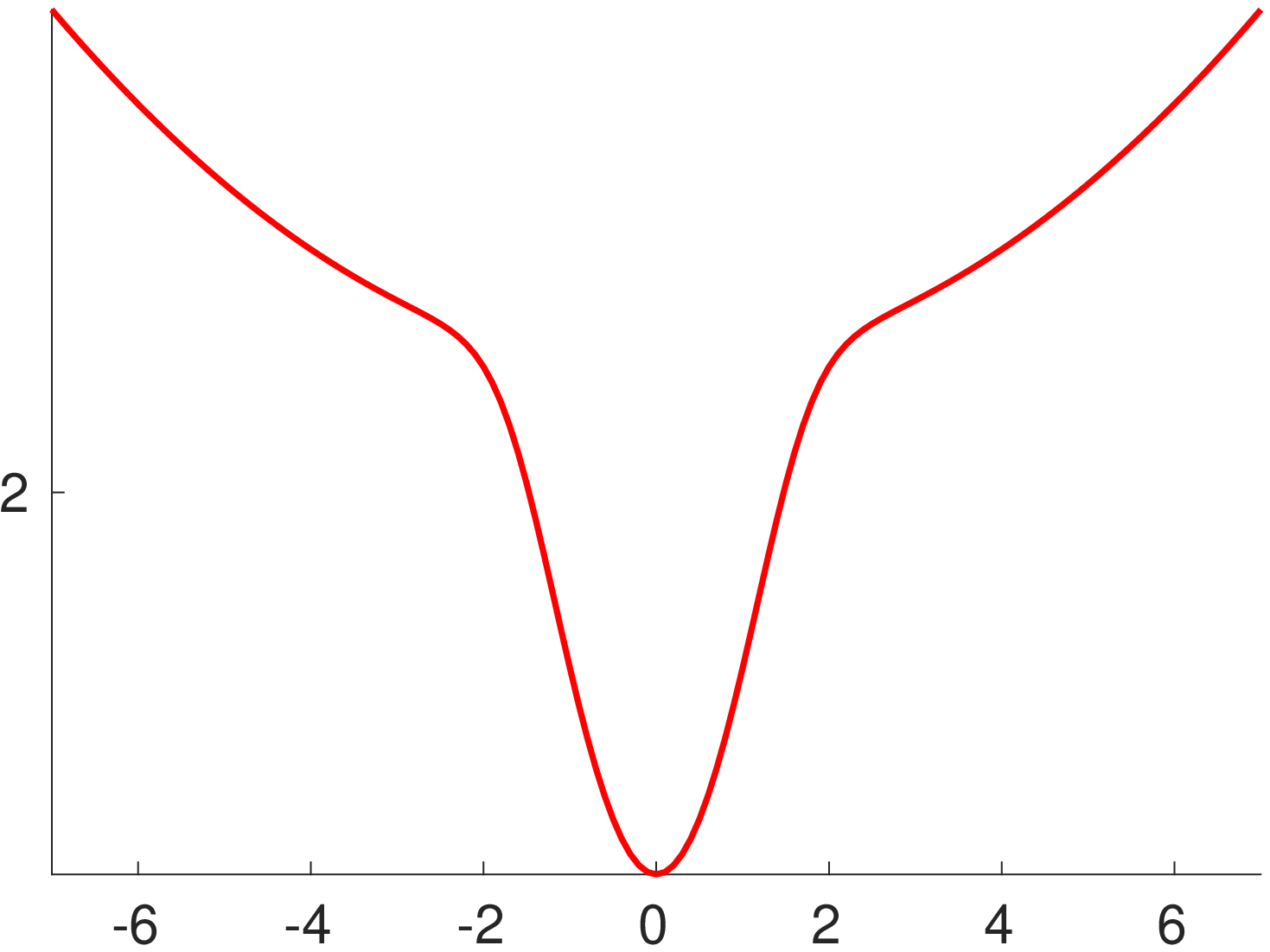}}
    \subfigure[Cauchy ($\delta=1$)]{\includegraphics[width=0.23\linewidth]{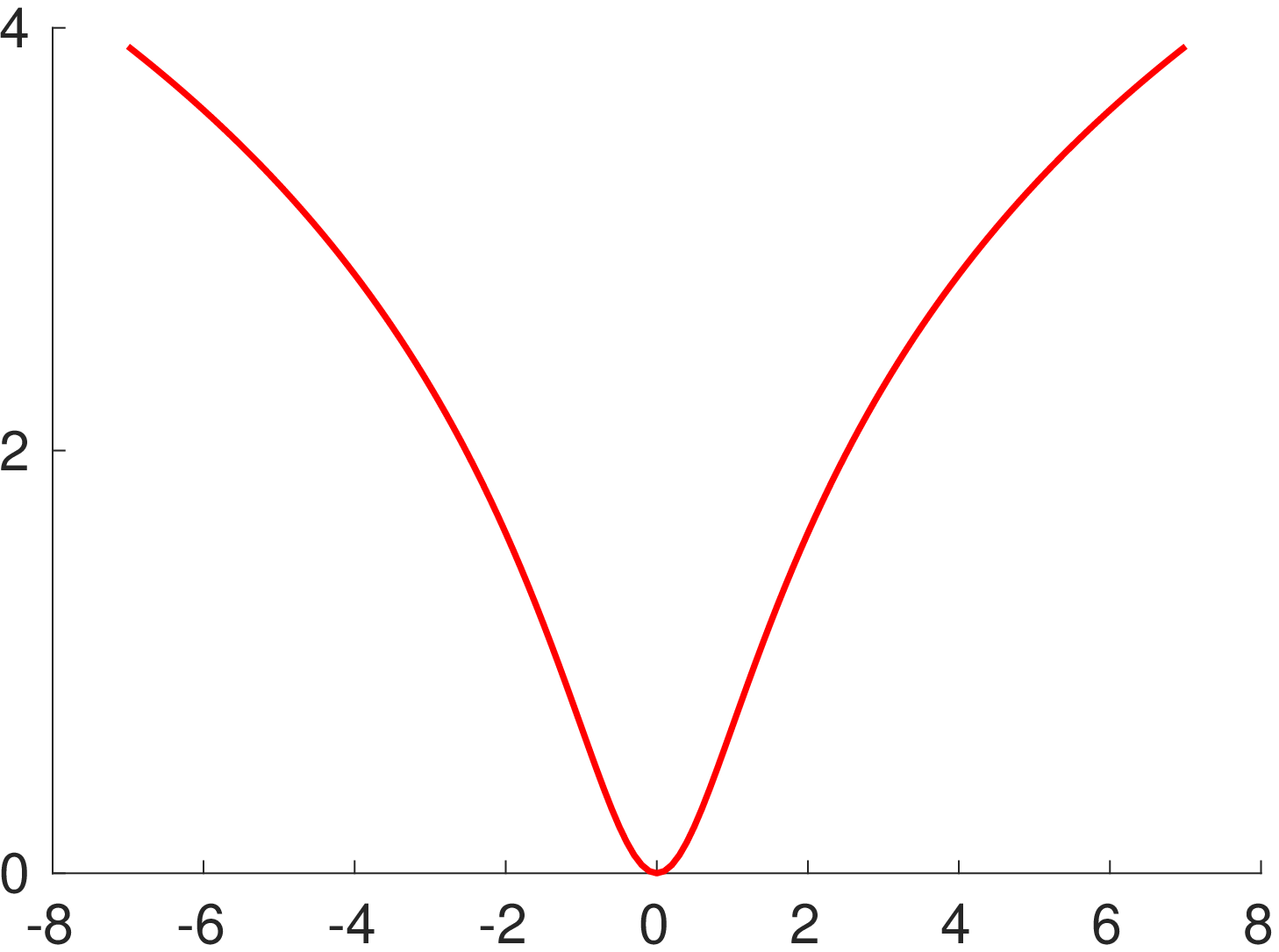}}
    \subfigure[Convex losses]{\includegraphics[width=0.23\linewidth]{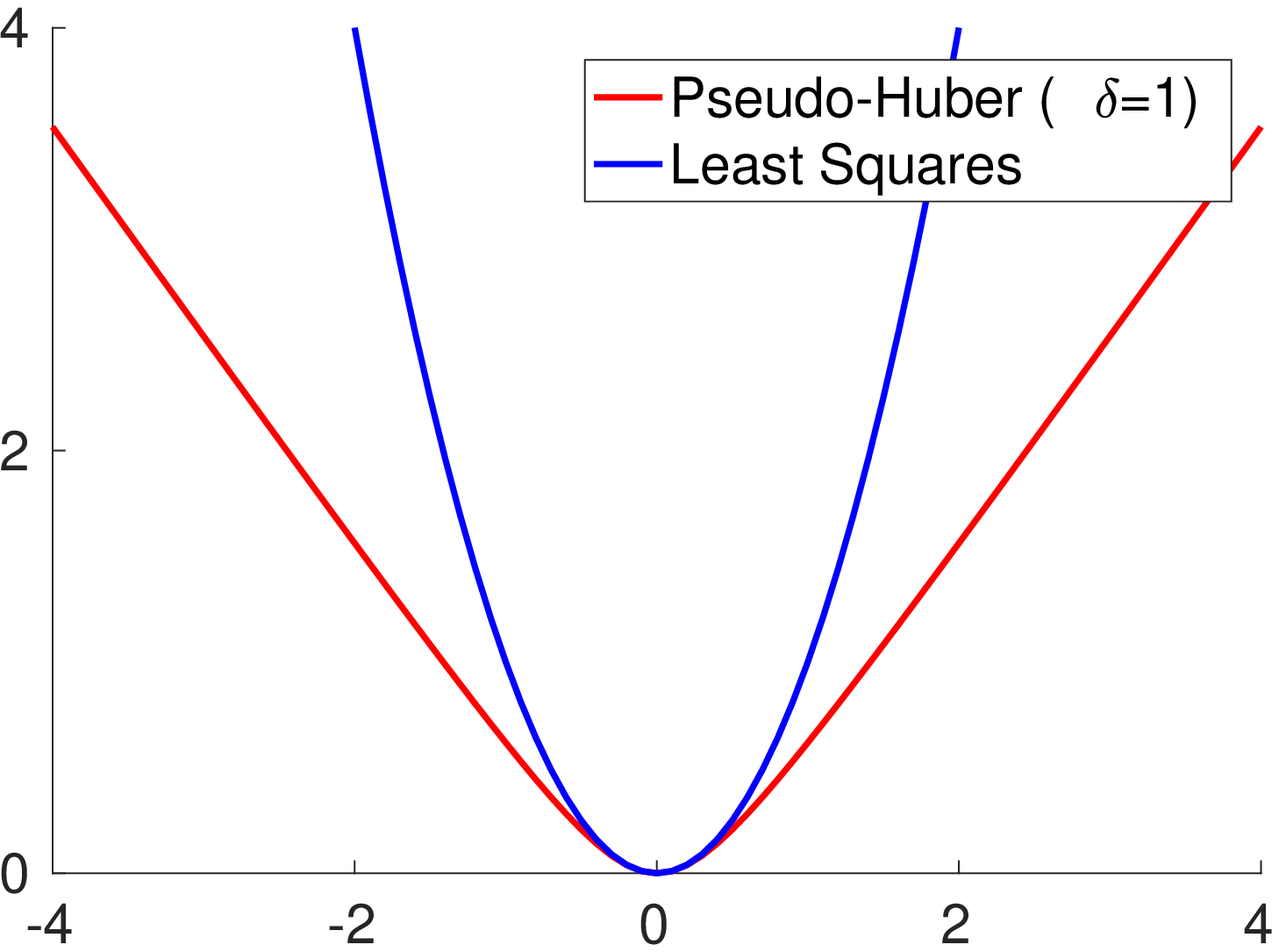}}
    \caption{Examples of convex and non-convex loss functions that satisfy Assumption \ref{ass}.}
\label{fig:example_of_loss}
\end{center}
\end{figure*}

As a motivation for our main result, we first analyze the case 
when the training samples are linearly independent, which requires $N\leq d+1.$ 
It can be seen as a generalization of Corollary 1 in \cite{Gori92}. 
\begin{theorem}\label{theo:independent_inputs}
    Let $\Phi:\P\to\RR$ be defined as in \eqref{eq:main1} and let the Assumptions \ref{ass} hold. 
    If the training samples are linearly independent, that is $\rank([X,\ones_N])=N$,
    then every critical point $(W_l,b_l)_{l=1}^L$ of $\Phi$ 
    for which the weight matrices $(W_l)_{l=2}^L$ have full column rank, that is 
    $\rank(W_l)=n_l$ for $l\in[2,L]$,  is a global minimum.
\end{theorem}
\ifpaper
\begin{proof}
  The proof is based on induction. At a critical point it holds $\nabla_{W_1}\Phi=X^T\Delta_1=0$
  and $\nabla_{b_1}\Phi=\Delta_1^T\ones_N=0$ thus $[X,\ones_N]^T\Delta_1=0.$
  By assumption, the data matrix $[X,\ones_N]^T\in\RR^{(d+1)\times N}$ has full column rank,
  this implies $\Delta_1=0.$
  Using induction, let us assume that $\Delta_k=0$ for some $1 \leq k \leq L-1$, then by Lemma \ref{lem:grad}, 
  we have $\Delta_k=(\Delta_{k+1}W_{k+1}^T) \circ \sigma'(G_{k})=0$. 
  As by assumption $\sigma'$ is strictly positive, this is equivalent to
  $\Delta_{k+1}W_{k+1}^T=0$ resp. $W_{k+1}\Delta_{k+1}^T=0$. 
  As by assumption $W_{k+1}$ has full column rank,
  it follows $\Delta_{k+1}=0$. Finally, we get $\Delta_L=0$. 
  With Lemma \ref{lem:grad} we thus get $l'(F_L-Y)\circ \sigma'(G_L)=0$ 
  which implies with the same argument as above $l'(F_L-Y)=0.$
  From our Assumption \ref{ass}, it holds that if $l'(a)=0$ then $a$ is a global minimum of $l.$
  Thus each individual entry of $(F_L-Y)$ must represent a global minimum of $l.$
  This combined with \eqref{eq:main1} implies that the critical point must be a global minimum of $\Phi.$
\end{proof}
\fi
Theorem \ref{theo:independent_inputs} implies that the weight matrices of potential saddle points 
or suboptimal local minima need to have low rank for one particular layer.
Note however that the set of low rank weight matrices in $\W$ has measure zero. 
At the moment we cannot prove that suboptimal low rank local minima cannot exist. However, it 
seems implausible that such suboptimal low rank local minima exist as every neighborhood of such points contains 
full rank matrices which increase the expressiveness of the network.
Thus it should be possible to use this degree of freedom to further reduce the loss, 
which contradicts the definition of a local minimum. 
Thus we conjecture that all local minima are indeed globally optimal.

The main restriction in the assumptions of Theorem \ref{theo:independent_inputs} 
is the linear independence of the training samples as it requires $N\leq d+1$, which is very restrictive in practice. 
We prove in this section a similar guarantee 
in our main Theorem \ref{theo:main} by implicitly transporting this condition to some higher layer. 
A similar guarantee has been proven by \cite{Yu95} for a single hidden layer network, 
whereas we consider general multi-layer networks.
The main ingredient of the proof of our main result is the observation in the following lemma.
\begin{lemma}\label{lem:zero_loss}
    Let $\Phi:\P\to\RR$ be defined as in \eqref{eq:main1} and let the Assumptions \ref{ass} hold. 
    Let $(W_l,b_l)_{l=1}^L \in\P$ be given. Assume there is some $k\in[L-1]$ s.t. the following holds
    \begin{enumerate}
	\item $\rank([F_k,\ones_N])=N$
	\item $\rank(W_l)=n_l, \, l\in[k+2,L]$
	\item $\nabla_{W_{k+1}} \Phi\Big( (W_l,b_l)_{l=1}^L \Big)=0$ \\
	      $\nabla_{b_{k+1}} \Phi\Big( (W_l,b_l)_{l=1}^L \Big)=0$ 
    \end{enumerate}
    then $(W_l,b_l)_{l=1}^L$ is a global minimum.
\end{lemma}
\ifpaper
\begin{proof}
    By Lemma \ref{lem:grad} it holds that 
    \begin{align*}
	\nabla_{W_{k+1}} \Phi =F_k^T\Delta_{k+1} =0, \quad
	\nabla_{b_{k+1}} \Phi =\Delta_{k+1}^T\ones_N =0, 
    \end{align*}
    which implies $[F_k,\ones_N]^T \Delta_{k+1}=0.$
    By our assumption, $\rank([F_k,\ones_N])=N$ it holds that $\Delta_{k+1}=0.$
    Since $\rank(W_l)=n_l, l\in[k+2,L]$, we can apply a similar induction argument
    as in the proof of Theorem \ref{theo:independent_inputs}, 
    to arrive at $\Delta_L=0$ and thus a global minimum.
\end{proof}
\fi
The first condition of Lemma \ref{lem:zero_loss} can be seen as a generalization of the requirement 
of linearly independent training inputs in Theorem \ref{theo:independent_inputs}
to a condition of linear independence of the feature vectors at a hidden layer. 
Lemma \ref{lem:zero_loss} suggests that if we want to make statements about the global optimality of critical points, 
it is sufficient to know when and which critical points fulfill these conditions.
The third condition is trivially satisfied by a critical point 
and the requirement of full column rank of the weight matrices is similar to  Theorem \ref{theo:independent_inputs}.
However, the first one may not be fulfilled since $\rank([F_k,\ones_N])$ is dependent not only on the weights but also on the architecture. 
The main difficulty of the proof of our following
main theorem is to prove that this first condition holds under the rather simple requirement that $n_k\geq N-1$
for a subset of all critical points.

But before we state the theorem we have to discuss a particular notion of non-degenerate critical point.
\begin{definition}[Block Hessian]
    Let $f:D\to\RR$ be a twice-continuously differentiable function defined on some open domain $D\subseteq\RR^n.$
    The Hessian w.r.t. a subset of variables $S\subseteq\Set{x_1,\ldots,x_n}$ 
    is denoted as $\nabla^2_S f(x)\in\RR^{|S|\times|S|}.$
    When $|S|=n$, we write $\nabla^2 f(x)\in\RR^{n\times n}$ to denote the full Hessian matrix.
\end{definition}
We use this to introduce a slightly more general notion of non-degenerate critical point.
\begin{definition}[Non-degenerate critical point]\label{def:non-degenerate}
    Let $f:D\to\RR$ be a twice-continuously differentiable function defined on some open domain $D\subseteq\RR^n.$
    Let $x\in D$ be a critical point, \ie $\nabla f(x)=0$, then
    \begin{itemize}
	  \item $x$ is non-degenerate for a subset of variables $S\subseteq\Set{x_1,\ldots,x_n}$
	  if $\nabla^2_S f(x)$ is non-singular.
	  \item $x$ is non-degenerate if $\nabla^2 f(x)$ is non-singular.
    \end{itemize}
\end{definition}
Note that a non-degenerate critical point might not be non-degenerate for a subset of variables, 
and vice versa, if it is non-degenerate on a subset of variables it does not necessarily imply non-degeneracy on the whole set.
For instance,
$$
\nabla^2f(x)=
\begin{array}{cc|cc}
    1&0&0&0\\ 
    0&1&0&0\\\hline
    0&0&0&0\\
    0&0&0&0\\
\end{array},\quad
\nabla^2f(y)=
\begin{array}{cc|cc}
    1&0&1&0\\ 
    0&1&0&1\\\hline
    1&0&0&0\\
    0&1&0&0\\
\end{array}
$$
Clearly, $\det{\nabla^2f(x)}=0$ but $\det{\nabla^2_{\Set{x_1,x_2}}}f(x)\neq 0,$
and $\det{\nabla^2f(y)}\neq 0$ but $\det{\nabla^2_{\Set{y_3,y_4}}}f(y)= 0.$
The concept of non-degeneracy on a subset of variables is crucial for the following statement of our main result.
\begin{theorem}\label{theo:main}
    Let $\Phi:\P\to\RR$ be defined as in \eqref{eq:main1} and let the Assumptions \ref{ass} hold. 
    Suppose $n_k\geq N-1$ for some $k\in[L-1].$
    Then every critical point $(W^*_l,b^*_l)_{l=1}^L$ of $\Phi$ which satisfies the following conditions
    \begin{enumerate} 
	\item $(W^*_l,b^*_l)_{l=1}^L$ is non-degenerate on $\Setbar{(W_l,b_l)}{l\in \mathcal{I}}$,
	for some subset $\mathcal{I}\subseteq\Set{k+1,\ldots,L}$ satisfying $\Set{k+1}\in\mathcal{I},$
	\item $(W^*_l)_{l=k+2}^L$ has full column rank, that is, $\rank(W^*_l)=n_l$ for $l\in[k+2,L]$,
   \end{enumerate}
    is a global minimum of $\Phi.$
\end{theorem}
First of all we note that the full column rank condition of $(W_l)_{l=k+2}^L$ in Theorem \ref{theo:independent_inputs}, 
and \ref{theo:main} implicitly requires that $n_{k+1}\geq n_{k+2}\geq\ldots\geq n_L.$
This means the network needs to have a pyramidal structure from layer $k+2$ to $L$. 
It is interesting to note that most modern neural network architectures have a pyramidal structure from some layer, 
typically the first hidden layer, on. 
Thus this is not a restrictive requirement. 
Indeed, one can even argue that Theorem \ref{theo:main} gives an implicit justification
as it hints on the fact that such networks are easy to train if one layer is sufficiently wide.

Note that Theorem \ref{theo:main} does not require fully non-degenerate critical points  
but non-degeneracy is only needed for some subset of variables that includes layer $k+1$. 
As a consequence of Theorem \ref{theo:main}, we get directly a stronger result for non-degenerate local minima.
\begin{corollary}\label{cor:nondegenerate_hessian}
    Let $\Phi:\P\to\RR$ be defined as in \eqref{eq:main1} and let the Assumptions \ref{ass} hold. 
    Suppose $n_k\geq N-1$ for some $k\in[L-1]$. 
    Then every non-degenerate local minimum $(W^*_l,b^*_l)_{l=1}^L$ of $\Phi$ for which
    $(W^*_l)_{l=k+2}^L$ has full column rank, that is $\rank(W^*_l)=n_l$, 
    is a global minimum of $\Phi.$
\end{corollary}
\ifpaper
\begin{proof}
The Hessian at a non-degenerate local minimum is positive definite and 
every principal submatrix of a positive definite matrix is again positive definite, 
in particular for the subset of variables $(W_l,b_l)_{l=k+1}^L$. 
Then application of Theorem \ref{theo:main} yields the result.
\end{proof}
\fi 

Let us discuss the implications of these results. First, note that Theorem \ref{theo:main} is slightly weaker than Theorem \ref{theo:independent_inputs} as it requires also non-degeneracy wrt to a set of variables including layer $k+1$. Moreover, similar to Theorem \ref{theo:independent_inputs}  it does not exclude the possibility of suboptimal local minima of
low rank in the layers ``above'' layer $k+1$. 
On the other hand it makes also very strong statements. 
In fact, if $n_k \geq N-1$ for some $k \in [L-1]$ then 
even degenerate saddle points/local maxima are excluded as long as they are non-degenerate with respect 
to any subset of parameters of upper layers that include layer $k+1$ and the rank condition holds. 
Thus given that  the weight matrices of the upper layers have full column rank , 
there is not much room left for degenerate saddle points/local maxima. 
Moreover, for a one-hidden-layer network for which $n_1\geq N-1$, 
\emph{every} non-degenerate critical point with respect to the output layer parameters is a global minimum, 
as the full rank condition is not active for one-hidden layer networks. 

Concerning the non-degeneracy condition of main Theorem \ref{theo:main}, 
one might ask how likely it is to encounter degenerate points of a smooth function. 
This is answered by an application of Sard's/Morse theorem in \cite{Mil1965}.
\begin{theorem}[A. Morse, p.11]
If $f:U \subset \RR^d \rightarrow \RR$ is twice continuously differentiable. Then for almost all $w \in \RR^d$ with respect to the Lebesgue measure it holds that
$f'$ defined as $f'(x)=f(x)+\inner{w,x}$ has only non-degenerate critical points.
\end{theorem}
Note that the theorem would still hold if one would draw $w$ uniformly at random from the set $\{z \in \RR^d\,|\, \norm{z}_2\leq \epsilon\}$ for any $\epsilon>0$. 
Thus almost every linear perturbation $f'$ of a function $f$ will lead to the fact all of its critical points are non-degenerate. 
Thus, this result indicates that exact degenerate points might be rare. 
Note however that in practice the Hessian at critical points can be close to singular (at least up to numerical precision), 
which might affect the training of neural networks negatively \cite{Levent17}.

As we argued for Theorem  \ref{theo:independent_inputs} our main Theorem \ref{theo:main} 
does not exclude the possibility of suboptimal degenerate local minima or suboptimal local minima of low rank. 
However, we conjecture that the second case cannot happen as every neighborhood of the local minima contains full rank
matrices which increase the expressiveness of the network and this additional flexibility 
can be used to reduce the loss which contradicts the definition of a local minimum.

As mentioned in the introduction the condition $n_k \geq N-1$ looks at first sight very strong. However, as mentioned in the introduction,
in practice often networks are used where one hidden layer is rather wide, that is $n_k$ is on the order of $N$ (typically it is the first layer of the network). As the condition of Theorem \ref{theo:main} is sufficient and not necessary, one can expect out of continuity reasons that the loss surface of networks where the condition is approximately true, is still rather well behaved,
in the sense that still most local minima are indeed globally optimal and the suboptimal ones are not far away from the globally optimal ones.



\section{Proof of Main Result}\label{sec:proof} 
For better readability, we first prove our main Theorem \ref{theo:main}
for a special case where $\mathcal{I}$ is the whole set of upper layers,
\ie $\mathcal{I}=\Set{k+1,\ldots,L},$
and then show how to extend the proof to the general case where $\mathcal{I}\subseteq\Set{k+1,\ldots,L}.$
Our proof strategy is as follows. 
We first show that the output of each layer are real analytic functions of network parameters.
Then we prove that there exists a set of parameters such that $\rank([F_k,\ones_N])=N.$
Using properties of real analytic functions, we conclude that the set of parameters where $\rank([F_k,\ones_N])<N$ has measure zero.
Then with the non-degeneracy condition, we can apply the implicit-function theorem to conclude that 
even if $\rank([F_k,\ones_N])=N$ is not true at a critical point, 
then still in any neighborhood of it there exists a point where the conditions of Lemma \ref{lem:zero_loss}
are true and the loss is minimal. 
By continuity of $\Phi,$ this implies that the loss must also be minimal at the critical point.

We introduce some notation frequently used in the proofs.
Let $B(x,r)=\{ z \in \RR^d\,|\, \norm{x-z}_2 < r\}$ be the open ball in $\RR^d$ of radius $r$ around $x$. 
\begin{lemma}\label{lem:analytic}
If the Assumptions \ref{ass} hold, then the output of each layer $f_l$ for every $l\in[L]$ 
are real analytic functions of the network parameters on $\mathcal{\P}.$
\end{lemma}
\begin{proof}
Any linear function is real analytic and 
the set of real analytic functions is closed under addition, multiplication and composition, 
see e.g. Prop. 2.2.2 and Prop. 2.2.8 in \cite{KraPar2002}.
As we assume that the activation function is real analytic, 
we get that all the output functions of the neural network $f_k$ 
are real analytic functions of the parameters as compositions of real analytic functions. 
\end{proof}

The concept of real analytic functions is important in our proofs
as these functions can never be ``constant'' in a set of the parameter space which has positive measure
unless they are constant everywhere. This is captured by the following lemma.
\begin{lemma}\cite{Dan15,Boris15}\label{lem:zeros_of_analytic}
    If $f:\RR^n\to\RR$ is a real analytic function which is not identically zero
    then the set $\Setbar{x\in\RR^n}{f(x)=0}$ has Lebesgue measure zero.
\end{lemma}
In the next lemma we show that there exist network parameters such that $\rank([F_k,\ones_N])=N$ holds if $n_k \geq N-1$. 
Note that this is only possible due to the fact that one uses non-linear activation functions. 
For deep linear networks, 
it is not possible for $F_k$ to achieve maximum rank if the layers below it are not sufficiently wide.
To see this, one considers $F_k=F_{k-1}W_k + \ones_N b_k^T$ for a linear network, 
then $\rank(F_k)\leq\min\{\rank(F_{k-1}), \rank(W_k)\}+1$ 
since the addition of a rank-one term does not increase the rank of a matrix by more than one.
By using induction, one gets $\rank(F_k)\leq \rank(W_l)+k-l+1$ for every $l\in[k].$

The existence of network parameters where $\rank([F_k,\ones_N])=N$ 
together  with the previous lemma will then be used to show that the set of network parameters 
where  $\rank([F_k,\ones_N])<N$ has measure zero.

%
%
\begin{lemma}\label{lem:exist}
    If the Assumptions \ref{ass} hold and $n_k\geq N-1$ for some $k\in[L-1]$, 
    then there exists at least one set of parameters $(W_l,b_l)_{l=1}^k$ such that $\rank([F_k,\ones_N])=N.$
\end{lemma}
\ifpaper
\begin{proof}
     We first show by induction that there always exists a set of parameters 
     $(W_l,b_l)_{l=1}^{k-1}$ s.t. $F_{k-1}$ has distinct rows.
    Indeed, we have $F_1=\sigma(XW_1+\ones_N  b_1^T)$.
    The set of $(W_1,b_1)$ that makes $F_1$ to have distinct rows is characterized by 
    \[ \sigma(W_1^T x_i+b_1)\neq\sigma(W_1^T x_j+b_1), \quad \forall i\neq j.\]
    Note, that $\sigma$ is strictly monotonic and thus bijective on its domain. Thus this is equivalent to
    \[ W_1^T (x_i-x_j) \neq 0, \quad \forall i \neq j.\]
    Let us denote the first column of $W_1$ by $a$, then the existence of $a$ for which
    \begin{align}\label{eq:hyperplanes}
     a^T(x_i-x_j) \neq 0, \quad \forall i \neq j,
    \end{align}
    would imply the result. Note that by assumption $x_i\neq x_j$ for all $i\neq j$.
    Then the set $\{ a \in \RR^d \,|\, a^T(x_i-x_j)=0\}$ is a hyperplane, which has measure zero
    and thus the set where condition \eqref{eq:hyperplanes} fails corresponds to the union of $\frac{N(N-1)}{2}$
    hyperplanes which again has measure zero. Thus there always exists a vector $a$ such that condition \eqref{eq:hyperplanes}
    is satisfied and thus there exists $(W_1,b_1)$ such that the rows of $F_1$ are distinct.
    Now, assume that $F_{p-1}$ has distinct rows for some $p\geq 1$, 
    then by the same argument as above we need to construct $W_p$ such that
    \[ W_p^T \Big(f_{p-1}(x_i)-f_{p-1}(x_j)\Big)\neq 0, \quad \forall i \neq j.\]
    By construction $f_{p-1}(x_i)\neq f_{p-1}(x_j)$ and thus with the same argument as above we can choose $W_p$ such that this condition holds. 
    As a result, there exists a set of parameters $(W_l,b_l)_{l=1}^{k-1}$ so that $F_{k-1}$ has distinct rows.
    
    Now, given that $F_{k-1}$ has distinct rows,
    we show how to construct $(W_k,b_k)$ in such a way that $[F_k,\ones_N]\in\RR^{N\times (n_k+1)}$ has full row rank.
    Since $n_k\geq N-1$, it is sufficient to make the first $N-1$ columns of $F_k$ together with the all-ones vector
    become linearly independent.
    In particular, let $F_k=[A,B]$ where $A\in\RR^{N\times (N-1)}$ and $B\in\RR^{N\times(n_k-N+1)}$
    be the matrices containing outputs of the first $(N-1)$ hidden units and last $(n_k-N+1)$ hidden units of layer $k$ respectively.
    Let $W_k=[w_1,\ldots,w_{N-1},w_N,\ldots,w_{n_k}]\in\RR^{n_{k-1}\times n_k}$
    and $b_k=[v_1,\ldots,v_{N-1},v_N,\ldots,v_{n_k}]\in\RR^{n_k}.$
    Let $Z=F_{k-1}=[z_1,\ldots,z_N]^T\in\RR^{N\times n_{k-1}}$ with $z_i\neq z_j$ for every $i\neq j.$
    By definition of $F_k$, it holds $A_{ij} = \sigma(z_i^T w_j+v_j)$ for $i\in[N],j\in[N-1].$
    As mentioned above, we just need to show there exists $(w_j,v_j)_{j=1}^{N-1}$ 
    so that $\rank([\ones_N,A])=N$ because then it will follow immediately that $\rank([F_k,\ones_N])=N.$
    Pick any $a\in\RR^{n_{k-1}}$ satisfying potentially after reordering w.l.o.g. 
    $\innerProd{a}{z_1}<\innerProd{a}{z_2}<\ldots<\innerProd{a}{z_N}.$
    By the discussion above such a vector always exists since the complementary set is contained in
    $\bigcup_{i\neq j} \Setbar{a\in\RR^{n_{k-1}}}{\innerProd{z_i-z_j}{a}=0}$ which has measure zero.
    
    We first prove the result for the case where $\sigma$ is bounded. Since $\sigma$ is bounded and strictly monotonically increasing, 
    there exist two finite values $\gamma,\mu\in\RR$ with $\mu < \gamma$ s.t. 
    $$
	\lim\limits_{\alpha\to-\infty} \sigma(\alpha)=\mu \quad \textrm{ and } \quad \\
	\lim\limits_{\alpha\to+\infty} \sigma(\alpha)=\gamma.
    $$
    Moreover, since $\sigma$ is strictly monotonically increasing 
	it holds for every $\beta\in\RR$, $\sigma(\beta)>\mu.$
	Pick some $\beta\in\RR$. 
	For $\alpha\in\RR$, we define $w_j=-\alpha a, v_j=\alpha z_j^T a + \beta$ for every $j\in[N-1].$
	Note that the matrix $A$ changes as we vary $\alpha$.
	Thus, we consider a family of matrices $A(\alpha)$ defined as
	$A(\alpha)_{ij}=\sigma(z_i^T w_j+v_j)=\sigma(\alpha (z_j-z_i)^T a + \beta).$ 
	Then it holds for every $i\in[N],j\in[N-1]$
	\begin{align*}
	    \lim\limits_{\alpha\to+\infty} A(\alpha)_{ij} = 
	    \begin{cases}\gamma&j>i\\ \sigma(\beta)&j=i\\ \mu&j<i\end{cases}
	\end{align*}
	Let $E(\alpha)=[\ones_N, A(\alpha)]$ then it holds
	\begin{align*}
	    \lim\limits_{\alpha\to+\infty} E(\alpha)_{ij} = 
	    \begin{cases}1&j=1\\ \mu&j\in[2,N],i\geq j\\ \sigma(\beta)&j=i+1\\ \gamma&\textrm{else}\end{cases}
	\end{align*}
	Let $\hat{E}(\alpha)$ be a modified matrix where one subtracts every row $i$ by row $(i-1)$ of $E(\alpha)$, 
	in particular, let 
	\begin{align}
	    \hat{E}(\alpha)_{ij} = \begin{cases}E(\alpha)_{ij}&i=1,j\in[N]\\ E(\alpha)_{ij}-E(\alpha)_{i-1,j}&i>1,j\in[N]\end{cases}
	\end{align}
	then it holds
	\begin{align*}
	    \lim\limits_{\alpha\to+\infty} \hat{E}(\alpha)_{ij} = 
	    \begin{cases}1&i=j=1\\ \mu-\sigma(\beta)< 0&i=j>1\\ 0&i>j \end{cases}
	\end{align*}
	We do not show the values of other entries as what matters is that the limit, 
	$\lim\limits_{\alpha\to+\infty}\hat{E}(\alpha)$, is an upper triangular matrix. 
	Thus, the determinant is equal to the product of its diagonal entries which is non-zero. 
	Note that the determinant of $\hat{E}(\alpha)$ is the same as that of $E(\alpha)$ 
	as subtraction of some row from some other row does not change the determinant, 
	and thus we get that $\lim\limits_{\alpha\to+\infty}E(\alpha)$ has full rank $N$.
  As the determinant of $E(\alpha)$ is a polynomial of its entries and thus continuous in $\alpha$, 
  there exists $\alpha_0\in\RR$ s.t. for every $\alpha\geq\alpha_0$ 
	it holds $\rank(E(\alpha))=\rank([\ones_N, A(\alpha)])=N.$
	Moreover, since $A$ is chosen as the first $(N-1)$ columns of $F_k$, 
	one can always choose the weights of the first $(N-1)$ hidden units of layer $k$ so that $\rank([F_k,\ones_N])=N$.

       In the case where the  activation function fulfills $|\sigma(t)|\leq \rho_1 e^{\rho_2 t}$ for $t< 0$ and 
       $|\sigma(t)|\leq \rho_3 t + \rho_4 $ for $t\geq 0$ we consider directly the determinant of the matrix $E(\alpha)$.
       In particular, let us pick some $\beta\in\RR$ such that $\sigma(\beta)\neq 0$.
       We consider the family of matrices $A(\alpha)$ defined as
       $A(\alpha)_{ij}=\sigma(z_i^T w_j+v_j)=\sigma(\alpha(z_j-z_i)^T a+\beta)$ 
       where $w_j=-\alpha a, v_j=\alpha z_j^T a + \beta$ for every $j\in[N-1].$ 
       Let $E(\alpha)=[A(\alpha), \ones_N].$ 
       Note that the all-ones vector is now situated at the last column of $E(\alpha)$ instead of first column as before.
       This column re-ordering does not change the rank of $E(\alpha).$
       By the Leibniz-formula one has
       \[ \det(E(\alpha)) = \sum_{\pi \in S_N} \mathrm{sign}(\pi) \prod_{j=1}^{N-1} E(\alpha)_{\pi(j) j},\]
       where $S_N$ is the set of all $N!$ permutations of the set $\{1,\ldots,N\}$ 
       and we used the fact that the last column of $E(\alpha)$ is equal to the all ones vector.
       Define the permutation $\gamma$ as $\gamma(j)=j$ for $j\in[N].$ 
       Then we have
       \begin{align*}
	  &\det(E(\alpha))\\ 
	  &= \mathrm{\sign}(\gamma) \sigma(\beta)^{N-1} + 
	  \sum_{\pi \in S_N \backslash \{\gamma\} } \mathrm{sign}(\pi) \prod_{j=1}^{N-1} E(\alpha)_{\pi(j) j}.
	\end{align*}
       The idea now is to show that $\prod_{j=1}^{N-1} E(\alpha)_{\pi(j) j}$ goes to zero 
       for every permutation $\pi\neq\gamma$ as $\alpha$ goes to infinity. 
       And since the whole summation goes to zero while $\sigma(\beta)\neq 0$, the determinant would be non-zero as desired.
       With that, we first note that for any permutation $\pi\neq \gamma$ there has to be at least one component $\pi(j)$ 
       where $\pi(j)>j$, in which case, $\delta_j=(z_j-z_{\pi(j)})^T a<0$ and 
       thus for sufficiently large $\alpha$, it holds $\alpha \delta_j + \beta<0$. 
       Thus \[ |E(\alpha)_{\pi(j) j}|=|\sigma(\alpha (z_j-z_{\pi(j)})^T a + \beta)| 
       \leq \rho_1 e^{\rho_2\beta} e^{-\alpha \rho_2 |\delta_j|}.\]
       If $\pi(j)=j$ then $E(\alpha)_{\pi(j) j}=\sigma(\beta).$
       In cases where $\pi(j)<j (j\neq N)$ it holds that $\delta_j=(z_j-z_{\pi(j)})^T a>0$ and thus for sufficiently large $\alpha$, 
       it holds $\alpha \delta_j + \beta>0$ and we have
      \[ |E(\alpha)_{\pi(j) j}|=|\sigma(\alpha (z_j-z_{\pi(j)})^T a + \beta)|\leq \rho_3 \delta_j \alpha + \rho_3 \beta +\rho_4.\]
      So far, we have shown that $|E(\alpha)_{\pi(j) j}|$ can always be upper-bounded 
      by an exponential function resp. affine function of $\alpha$ when $\pi(j)>j$ resp. $\pi(j)<j$ or 
      it is just a constant when $\pi(j)=j.$
      The above observations imply that there exist positive constants $P,Q,R,S,T$ 
      such that it holds for every $\pi\in S_N\setminus\Set{\gamma},$
      \[  \Big|\prod_{j=1}^{N-1} E(\alpha)_{\pi(j) j}\Big| \; \leq \; R (P \alpha + Q)^{S} e^{-\alpha T}.\]
      As $\alpha \rightarrow \infty$ the upper bound goes to zero.
      As there are only finitely many such terms, we get
      \[ \lim_{\alpha \rightarrow \infty}  \det(E(\alpha)) = \mathrm{\sign}(\gamma) \sigma(\beta)^{N-1} \neq 0,\]
      and thus with the same argument as before we can argue that there exists a finite $\alpha_0$ 
      for which $E(\alpha)$ has full rank.

\end{proof}
\fi
Now we combine the previous lemma with Lemma \ref{lem:zeros_of_analytic} to conclude the following.

\begin{lemma}\label{lem:zero_measure}
    If the Assumptions \ref{ass} hold and $n_k\geq N-1$ for some $k\in[L-1]$ then 
    the set $S\bydef\Setbar{\big(W_l,b_l\big)_{l=1}^k}{\rank([F_k,\ones_N])<N}$
    has Lebesgue measure zero.
\end{lemma}
\ifpaper
\begin{proof}
    Let $E_k=[F_k,\ones_N]\in\RR^{N\times (n_k+1)}$. 
    Note that with Lemma \ref{lem:analytic} the output $F_k$ of layer $k$ is an analytic function of 
    the network parameters on $\mathcal{\P}$. The set of low rank matrices $E_k$ can be characterized by a system of equations 
    such that the $\binom{n_k+1}{N}$ determinants of all $N \times N$ submatrices of $E_k$ are zero. 
    As the determinant is a polynomial in the entries of the matrix
    and thus an analytic function of the entries and composition of analytic functions are again analytic, we conclude that each determinant is an analytic
    function of the network parameters of the first $k$ layers. 
    By Lemma \ref{lem:exist} there exists at least one set of network parameters  of the first $k$ layers 
    such that one of these determinant functions
    is not identically zero and thus by Lemma \ref{lem:zeros_of_analytic} the set of network parameters where this determinant is zero has measure zero.
    But as all submatrices need to have low rank in order that $\rank([F_k,\ones_N])<N$, 
    it follows that the set of network parameters where $\rank([F_k,\ones_N])<N$ has measure zero.   
\end{proof}
\fi
We conclude that for $n_k\geq N-1$ even if there are network parameters such that $\rank([F_k,\ones_N])<N$, then \emph{every} neighborhood of these
parameters contains network parameters such that $\rank([F_k,\ones_N])=N.$
\begin{corollary}\label{cor:perturb}
   If the Assumptions \ref{ass} hold and $n_k\geq N-1$ for some $k\in[L-1]$, 
   then for any given $(W^0_l,b^0_l)_{l=1}^k$ and 
    for every $\epsilon>0$, 
    there exists at least one
    $\big(W_l,b_l\big)_{l=1}^k\in B\Big(\big(W^0_l,b^0_l\big)_{l=1}^k,\epsilon\Big)$ 
    s.t. $\rank([F_k,\ones_N])=N.$
\end{corollary}
\begin{proof}
    Let $S\bydef\Setbar{\big(W_l,b_l\big)_{l=1}^k}{\rank([F_k,\ones_N])<N}.$
    The ball $B\Big(\big(W_l,b_l\big)_{l=1}^k,\epsilon\Big)$ has positive Lebesgue measure 
    while $S$ has measure zero due to Lemma \ref{lem:zero_measure}. 
    Thus, for every $\big(W_l,b_l\big)_{l=1}^k \in B\Big(\big(W^0_l,b^0_l\big)_{l=1}^k,\epsilon\Big) \setminus S$ 
    it holds $\rank([F_k,\ones_N])=N.$
\end{proof}
The final proof of our main Theorem \ref{theo:main} is heavily based on the implicit function theorem, see \eg \cite{Marsden74}.
\begin{theorem}\label{theo:IFT}
	Let $\Psi:\RR^s\times\RR^t\to\RR^t$ be a continuously differentiable function.
	Suppose $(u_0,v_0)\in\RR^s\times\RR^t$ and $\Psi(u_0,v_0)=0.$
	If the Jacobian matrix w.r.t. $v$, 
	$$J_v\Psi(u_0,v_0) = 
	    \begin{bmatrix}
		\frac{\partial\Psi_1}{\partial v_1}&\cdots&\frac{\partial\Psi_1}{\partial v_t}\\              
		\vdots&&\vdots\\
		\frac{\partial\Psi_t}{\partial v_1}&\cdots&\frac{\partial\Psi_t}{\partial v_t}
	    \end{bmatrix}
	    \in\RR^{t\times t}
	$$
	is non-singular at $(u_0,v_0)$, then there is an open ball $B(u_0,\epsilon)$ for some $\epsilon>0$ 
	and a unique function $\alpha:B(u_0,\epsilon)\to\RR^t$ such that
	$\Psi(u,\alpha(u))=0$ for all $u\in B(u_0,\epsilon)$.
	Furthermore, $\alpha$ is continuously differentiable.
\end{theorem}
    
With all the intermediate results proven above, we are finally ready for the proof of the main result.
\paragraph{Proof of Theorem \ref{theo:main} for case $\mathcal{I}=\Set{k+1,\ldots,L}$}$\;$\\
	Let us divide the set of all parameters of the network into two subsets where
	one corresponds to all parameters of all layers up to $k$, 
	for that we denote $u=[\vec(W_1)^T,b_1^T,\ldots,\vec(W_k)^T,b_k^T]^T$, 
	and the other corresponds to the remaining parameters,
	for that we denote $v=[\vec(W_{k+1})^T,b_{k+1}^T,\ldots,\vec(W_L)^T,b_L^T]^T.$
	By abuse of notation, we write $\Phi(u,v)$ to denote $\Phi\Big( (W_l,b_l)_{l=1}^L \Big).$
	Let $s=\dim(u), t=\dim(v)$ and $(u^*,v^*)\in\RR^s\times\RR^t$ 
	be the corresponding vectors for the critical point $(W^*_l,b^*_l)_{l=1}^L.$
	Let $\Psi:\RR^s\times\RR^t\to\RR^t$ be a map defined as
	$\Psi(u,v)=\nabla_v \Phi(u,v) \in \RR^t$, 
	which is the gradient mapping of $\Phi$ w.r.t. all parameters of the upper layers from $(k+1)$ to $L.$
	Since the gradient vanishes at a critical point, it holds that $\Psi(u^*,v^*)=\nabla_v \Phi(u^*,v^*) =0.$
	The Jacobian of $\Psi$ w.r.t. $v$ is the principal submatrix of the Hessian of $\Phi$ w.r.t. $v$,
	that is, $J_v\Psi(u,v)=\nabla^2_v\Phi(u,v) \in \RR^{t\times t}.$
	As the critical point is assumed to be non-degenerate with respect to $v$, it holds that $J_v\Psi(u^*,v^*)=\nabla^2_v\Phi(u^*,v^*)$ is non-singular.
	Moreover, $\Psi$ is continuously differentiable since $\Phi\in C^2(\mathcal{\P})$ due to Assumption \ref{ass}.
	Therefore, $\Psi$ and $(u^*,v^*)$ satisfy the conditions of the implicit function theorem \ref{theo:IFT}. 
	Thus there exists an open ball $B(u^*, \delta_1)\subset\RR^s$ 
	for some $\delta_1>0$ and a continuously differentiable function
	$\alpha:B(u^*,\delta_1)\to\RR^t$ such that
	$$\begin{cases}
	    \Psi(u,\alpha(u)) = 0, & \forall\, u\in B(u^*, \delta_1)\\
	    \alpha(u^*) = v^*
	\end{cases}$$
	By assumption we have $\rank(W^*_l)=n_l, l\in[k+2,L]$, 
	that is the weight matrices of the ``upper'' layers have full column rank.
	Note that $(W^*_l)_{l=k+2}^L$ corresponds to the weight matrix part of $v^*$ where one leaves out $W^*_{k+1}$.
	Thus there exists a sufficiently small $\epsilon$ such that for any $v \in B(v^*,\epsilon)$, the weight matrix part $(W_l)_{l=k+2}^L$ of $v$
	has full column rank. 
	In particular, this, combined with the continuity of $\alpha$,
	implies that for a potentially smaller $0<\delta_2 \leq \delta_1$, it holds for all $u \in B(u^*,\delta_2)$ that
	\[ \Psi(u,\alpha(u))=0, \; \alpha(u^*)=v^*,\]
	and that the weight matrix part $(W_l)_{l=k+2}^L$ of $\alpha(u) \in \RR^t$ has full column rank.
	
	Now, by Corollary \ref{cor:perturb} for any $0<\delta_3 \leq \delta_2$ there exists a $\tilde{u} \in B(u^*,\delta_3)$
	such that the generated output matrix $\tilde{F}_k$ at layer $k$ 
	of the corresponding network parameters of $\tilde{u}$ satisfies $\rank([\tilde{F_k}, \ones_N])=N.$
	Moreover, it holds for $\tilde{v}=\alpha(\tilde{u})$ that $\Psi(\tilde{u},\tilde{v})=0$ 
	and the weight matrix part $(\tilde{W}_l)_{l=k+2}^L$ of $\tilde{v}$ has full column rank.
	Assume $(\tilde{u}, \tilde{v})$ corresponds to the following representation
	$$\begin{cases}
	    \tilde{u}= [\vec(\tilde{W}_1)^T,\tilde{b}_1^T,\ldots,\vec(\tilde{W}_k)^T,\tilde{b}_k^T]^T \in\RR^s\\
	    \tilde{v}= [\vec(\tilde{W}_{k+1})^T,\tilde{b}_{k+1}^T,\ldots,\vec(\tilde{W}_L)^T,\tilde{b}_L^T]^T  \in\RR^t\\
	\end{cases}$$
	We obtain the following
	$$\begin{cases}
	    \Psi(\tilde{u},\tilde{v})=0 \Rightarrow \nabla_{W_{k+1}}\Phi\Big( (\tilde{W_l},\tilde{b_l})_{l=1}^k\Big)=0 \\
	    \Psi(\tilde{u},\tilde{v})=0 \Rightarrow \nabla_{b_{k+1}}\Phi\Big( (\tilde{W_l},\tilde{b_l})_{l=1}^k\Big)=0 \\
	    \rank(\tilde{W_l})=n_l, \forall\,l\in[k+2,L] \\
	    \rank([\tilde{F_k}, \ones_N]) = N
	\end{cases}$$
	Thus, Lemma \ref{lem:zero_loss} implies that $(\tilde{W_l},\tilde{b_l})_{l=1}^L$ is a global minimum of $\Phi.$
	Let $p^* = \Phi\Big( (\tilde{W}_l,\tilde{b}_l)_{l=1}^L \Big) = \Phi(\tilde{u},\tilde{v}).$
	Note that this construction can be done for any $\delta_3\in(0,\delta_2].$
	In particular, let $(\gamma_r)_{r=1}^\infty$ be a strictly monotonically decreasing sequence
	such that $\gamma_1=\delta_3$ and $\lim_{r\rightarrow\infty}\gamma_r=0.$
	By Corollary \ref{cor:perturb} and the previous argument, 
	we can choose for any $\gamma_r>0$ a point $\tilde{u_r}\in B(u^*,\gamma_r)$ such that 
	$\tilde{v}_r=\alpha(\tilde{u}_r)$ has full rank and $\Phi(\tilde{u}_r,\tilde{v}_r)=p^*.$
	Moreover, as $\lim_{r\rightarrow\infty}\gamma_r=0$, it follows that $\lim_{r\rightarrow \infty} \tilde{u}_r=u^*$ 
	and as $\alpha$ is a continuous function, it holds with $\tilde{v}_r=\alpha(\tilde{u}_r)$ 
	that $\lim_{r\rightarrow \infty} \tilde{v}_r=\lim_{r\rightarrow \infty} \alpha(\tilde{u}_r)=\alpha(\lim_{r\rightarrow \infty} \tilde{u}_r)=\alpha(u^*)=v^*$.
	Thus we get $\lim_{r \rightarrow \infty} (\tilde{u}_r,\tilde{v}_r) = (u^*,v^*)$ 
	and as $\Phi$ is a continuous function it holds
	\[ \lim_{r\rightarrow \infty} \Phi\Big( (\tilde{u}_r,\tilde{v}_r)\Big)=\Phi(u^*,v^*) = p^*,\]
	as $\Phi$ attains the global minimum for the whole sequence $(\tilde{u}_r,\tilde{v}_r).$	
\paragraph{Proof of Theorem \ref{theo:main} for general case}$\;$\\
    In the general case $\mathcal{I}\subseteq\Set{k+1,\ldots,L}$, 
    the previous proof can be easily adapted.
    The idea is that we fix all layers in $\Set{k+1,\ldots,L}\setminus\mathcal{I}.$
    In particular, let
    \begin{align*}
	\begin{cases}
	    u=[\vec(W_1)^T,b_1^T,\ldots,\vec(W_k)^T,b_k^T]^T \\
	    v=[\vec(W_{\mathcal{I}(1)})^T,b_{\mathcal{I}(1)}^T,\ldots,\vec(W_{\mathcal{I}(|\mathcal{I}|)})^T,b_{\mathcal{I}(|\mathcal{I}|)}^T]^T.
	\end{cases}
    \end{align*}
    Let $s=\dim(u), t=\dim(v)$ and $(u^*,v^*)\in\RR^s\times\RR^t$ 
    be the corresponding vectors at $(W^*_l,b^*_l)_{l=1}^L.$
    Let $\Psi:\RR^s\times\RR^t\to\RR^t$ be a map defined as 
    $\Psi(u,v)=\nabla_v \Phi\Big((W_l,b_l)_{l=1}^L\Big)$
    with $\Psi(u^*,v^*)=\nabla_v \Phi\Big( (W^*_l,b^*_l)_{l=1}^L\Big)=0.$
    
    The only difference is that all the layers from 
    $\Set{k+1,\ldots,L}\setminus\mathcal{I}$ are hold fixed.
    They are not contained in the arguments of $\Psi$, thus will not be involved in our perturbation analysis.
    In this way, the full rank property of the weight matrices of these layers are preserved, 
    which is needed to obtain the global minimum.

\section{Relaxing the Condition on the Number of Hidden Units}
We have seen that $n_k\geq N-1$ is a sufficient condition which leads to a rather simple structure of the critical points, in the sense
that all local minima which have full rank in the layers $k+2$ to $L$ 
and for which the Hessian is non-degenerate on any subset of upper layers that includes layer $k+1$ 
are automatically globally optimal. This suggests that suboptimal locally optimal points are either completely
absent or relatively rare. We have motivated before that networks with a certain wide layer are used
in practice, which shows that the condition $n_k\geq N-1$ is not completely unrealistic. On the other hand we want to discuss in this section
how it could be potentially relaxed. The following result will provide 
some intuition about the case $n_k<N-1$, but will not be as strong
as our main result \ref{theo:main} which makes statements about a large class of critical points. 
The main idea is that with the condition $n_k\geq N-1$
the data is linearly separable at layer $k$. As modern neural networks are expressive enough to represent any function, see \cite{Zhang16} for an
interesting discussion on this, one can expect that in some layer the training data becomes linearly separable.
We prove that any critical point, for which the ``learned'' network outputs at any layer are linearly separable 
(see Definition \ref{def:linearly_separable}) is a global minimum of the training error.

\begin{definition}[Linearly separable vectors]\label{def:linearly_separable}
    A set of vectors $(x_i)_{i=1}^N\in\RR^d$ from $m$ classes $(C_j)_{j=1}^m$ 
    is called linearly separable if there exist $m$ vectors $(a_j)_{j=1}^m\in\RR^d$
    and $m$ scalars $(b_j)_{j=1}^m\in\RR$ so that 
    $a_j^T x_i+b_j >0$ for $x_i\in C_j$ and $a_j^T x_i+b_j <0$ for $x_i\notin C_j$
    for every $i\in[N],j\in[m].$
\end{definition}


In this section, we use a slightly different loss function than in the previous section. The reason
is that the standard least squares loss is not necessarily small when the data is linearly separable.
Let $C_1,\ldots,C_m$ denote $m$ classes.
We consider the objective function $\Phi:\P\to\RR$ from \eqref{eq:main1}
\begin{align}\label{eq:main2}
    \Phi\Big( (W_l,b_l)_{l=1}^L\Big) = \sum_{i=1}^N \sum_{j=1}^m   l\big(f_{Lj}(x_i) - y_{ij}\big)
\end{align}
where the loss function now takes the new form
\begin{align*}
   l\big(f_{Lj}(x_i) - y_{ij}\big) =
   \begin{cases}
	l_1\big(f_{Lj}(x_i) - y_{ij}\big) & x_i\in C_j \\
	l_2\big(f_{Lj}(x_i) - y_{ij}\big) & x_i\notin C_j 
   \end{cases} 
\end{align*}
where $l_1,l_2$ penalize the deviation from the label encoding for the true class resp. wrong classes.
We assume that the minimum of $\Phi$ is attained over $\P.$
Note that $\Phi$ is bounded from below by zero as $l_1$ and $l_2$ are non-negative loss functions.
The results of this section are made under the following assumptions on the activation and loss function.
\begin{assumptions}\label{ass2}
\begin{enumerate}
    \item $\sigma\in C^1(\RR)$ and strictly monotonically increasing.
    \item $l_1:\RR\to\RR_+, l_1\in C^1$, $l_1(a)=0\Leftrightarrow a\geq0$, $l_1'(a)=0\Leftrightarrow a\geq0$ and $l_1'(a)<0\,\,\forall\,a<0$ 
    \item $l_2:\RR\to\RR_+, l_2\in C^1$, $l_2(a)=0\Leftrightarrow a\leq0$, $l_2'(a)=0\Leftrightarrow a\leq0$ and $l_2'(a)>0\,\,\forall\,a>0$ 
\end{enumerate}
\end{assumptions}
In classification tasks, this loss function encourages higher values for the true class and lower values for wrong classes.
An example of the loss function that satisfies Assumption \ref{ass2} is given as (see Figure \ref{fig:error_function}): 
\begin{align*}
    l_1(a)=
    \begin{cases}
	a^2 & a\leq0\\
	0 & a\geq0
    \end{cases} &\quad&
    l_2(a)=
    \begin{cases}
	0 & a\leq0\\
	a^2 & a\geq0
    \end{cases}
\end{align*}
Note that for a $\{+1,-1\}$-label encoding, $+1$ for the true class and $-1$ for all wrong classes, 
one can rewrite \eqref{eq:main2} as
\[ \Phi\Big( (W_l,b_l)_{l=1}^L\Big) = \sum_{i=1}^N \sum_{j=1}^m \max\{0,1-y_{ij} f_{Lj}(x_i)\}^2,\]
which is similar to the truncated squared loss (also called squared hinge loss) used in the SVM for binary classification.
\begin{figure}
    \center
    \begin{tikzpicture}
	  \draw[->] (-2.2,0) -- (2.2,0) node[right] {};
	  \draw[->] (0,-.5) -- (0,2) node[above] {};
	  \draw[scale=1,domain=-1.5:0,smooth,variable=\x,blue] plot ({\x},{0});
	  \draw[scale=1,domain= 0:1.2,smooth,variable=\x,blue] plot ({\x},{\x*\x});
	  \draw[scale=1,domain=-1.2:0,smooth,variable=\x,red]  plot ({\x},{\x*\x});
	  \draw[scale=1,domain= 0:1.5,smooth,variable=\x,red]  plot ({\x},{0});
	  \node at (-1.2,.8) {\footnotesize $l_1$};
	  \node at (1.2,.8) {\footnotesize $l_2$};
    \end{tikzpicture}
    \caption{An example of $l_1,l_2.$}
    \label{fig:error_function}
\end{figure}
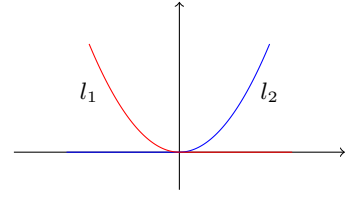
Since $\sigma$ and $l$ are continuously differentiable, 
all the results from Lemma \ref{lem:grad} still hold.

Our main result in this section is stated as follows.
\begin{theorem}\label{theo:main2}
    Let $\Phi:\P \to\RR_+$ be defined as in \eqref{eq:main2} and let the Assumptions \ref{ass2} hold.
    Then it follows:
    \begin{enumerate}
	\item Every critical point of $\Phi$ 
	for which the feature vectors contained in the rows of $F_k$ are linearly separable 
	and all the weight matrices $(W_l)_{l=k+2}^L$ have full column rank is a global minimum.
	
	\item If the training inputs are linearly separable then every critical point of $\Phi$ 
	for which all the weight matrices $(W_l)_{l=2}^L$ have full column rank is a global minimum.
    \end{enumerate}
\end{theorem}
\ifpaper
\begin{proof}
    \begin{enumerate}
	\item 
	Let $\tilde{F_k}=[F_k,\ones_N].$
	Since $F_k$ contains linearly separable feature vectors, 
	there exists $m$ vectors $h_1,\ldots,h_m\in\RR^{n_k+1}$
	s.t. $\innerProd{h_j}{{(\tilde{F_k})}_{i:}}>0$ for $x_i\in C_j$ and 
	$\innerProd{h_j}{{(\tilde{F_k})}_{i:}}<0$ for $x_i\notin C_j.$
	Let $H=[h_1,\ldots,h_m]\in\RR^{(n_k+1)\times m},$ one obtains
	\begin{align*}
	    {(H^T \tilde{F_k}^T)}_{ji} = \innerProd{h_j}{{(\tilde{F_k})}_{i:}} \begin{cases}>0 & x_i\in C_j \\ <0 & x_i\notin C_j \end{cases} .
	\end{align*}
	On the other hand,
	\begin{align*}
	    {(\Delta_L)}_{ij} 
	    &= \delta_{Lj}(x_i)\\
	    &= \frac{\partial\Phi}{\partial g_{Lj}(x_i)}  \nonumber\\
	    &=\begin{cases}
		l_1'(f_{Lj}(x_i)-y_{ij})\sigma'(g_{Lj}(x_i)) & x_i\in C_j \\ 
		l_2'(f_{Lj}(x_i)-y_{ij})\sigma'(g_{Lj}(x_i)) & x_i\notin C_j 
	    \end{cases} 
	\end{align*}
	We show that $H^T \tilde{F_k}^T \Delta_L=0$ if and only if $\Delta_L=0.$
	Indeed, if $\Delta_L=0$ the implication is trivial. For the other direction, assume that $H^T \tilde{F_k}^T \Delta_L=0.$
	Then it holds for every $j\in[m]$ that $0=(H^T \tilde{F_k}^T \Delta_L)_{jj}=\sum_{i=1}^N (H^T \tilde{F_k}^T)_{ji} (\Delta_L)_{ij}$.
	In particular,
	\begin{align*}
	  &\sum_{i=1}^N (H^T \tilde{F_k}^T)_{ji} (\Delta_L)_{ij}\\
	= &\sum_{\stackrel{i}{x_i \in C_j}} \inner{(\tilde{F_k})_{i,:},h_j}l_1'(f_{Lj}(x_i)-y_{ij})\sigma'(g_{Lj}(x_i))  \\
	 +&\sum_{\stackrel{i}{x_i \notin C_j}} \inner{(\tilde{F_k})_{i,:},h_j}l_2'(f_{Lj}(x_i)-y_{ij})\sigma'(g_{Lj}(x_i))\leq 0
	\end{align*}
	Note that under the assumptions on the loss and activation function and since the features are separable, the terms in both
	sums are non-positive and thus the sum can only vanish if all terms vanish which implies 
	\begin{align}\label{eq:loss-sep}
	    \begin{cases}
		l_1'\big(f_{Lj}(x_i) - y_{ij}\big)=0 & x_i\in C_j\\
		l_2'\big(f_{Lj}(x_i) - y_{ij}\big)=0 & x_i\notin C_j 
	    \end{cases} \forall i\in[N],j\in[m]
	\end{align} 
	which yields $\Delta_L=0.$
	
	Back to the main proof, the idea is to prove that $\Delta_L=0$ at the given critical point.
	Let us assume for the sake of contradiction that $\Delta_L\neq 0$.
	For every $l\in[L], i_l\in[n_l]$ 
	define $\Sigma_{i_l}^l=\mathop{diag}(\sigma'(g_{l i_l}(x_1)), \ldots, \sigma'(g_{l i_l}(x_N))).$
	Since the cumulative product $\prod_{l=k+1}^{L-1}\Sigma_{i_l}^l$ is a $N\times N$ diagonal matrix which
	contains only positive entries in its diagonal, it does not change the sign pattern of $\Delta_L$, 
	and thus it holds with, $H^T \tilde{F_k}^T \Delta_L=0$ if and only if $\Delta_L=0$, 
	for every $(i_{k+1},\ldots,i_{L-1})\in[n_{k+1}]\times\ldots\times[n_{L-1}]$ that
	\begin{align}\label{eq:induction}
	    0&\neq H^T \tilde{F_k}^T \Big(\prod_{l=k+1}^{L-1}\Sigma_{i_l}^l\Big) \Delta_L \\
	    0&\neq H^T \tilde{F_k}^T \Big(\prod_{l=k+1}^{L-1}\Sigma_{i_l}^l\Big) \Delta_L W_{L}^T & (\rank(W_L)=n_L) \nonumber,
	\end{align}
	where the last inequality is implied by \eqref{eq:induction} as $W_L$ has full column rank $n_L=m$. 
	Since the above product of matrices is a non-zero matrix, there must exist a non-zero column, say $p\in[n_{L-1}]$, then
	\begin{align*}
	    0\neq H^T \tilde{F_k}^T \Big(\prod_{l=k+1}^{L-1}\Sigma_{i_l}^l\Big) \Big(\Delta_L W_{L}^T\Big)_{:p} 
	\end{align*}
	Since $i_{L-1}$ is arbitrary, pick $i_{L-1}=p$ one obtains
	\begin{align}\label{eq1}
	    0&\neq H^T \tilde{F_k}^T \Big(\prod_{l=k+1}^{L-2}\Sigma_{i_l}^l\Big) \underbrace{\Sigma_p^{L-1}\Big(\Delta_L W_{L}^T\Big)_{:p}}_{(\Delta_{L-1})_{:p}} 
	\end{align}  
	Moreover, it holds for every $i\in[N]$
	\begin{align*}
	    &\Big( \Sigma_p^{L-1} \big(\Delta_L W_{L}^T\big)_{:p} \Big)_i \\
	    &=\sigma'(g_{(L-1)p}(x_i)) \sum_{j=1}^{n_{L}} \delta_{Lj}(x_i) {(W_L)}_{pj} \\
	    &=\delta_{(L-1)p}(x_i) \\
	    &=\big(\Delta_{L-1}\big)_{ip} 
	\end{align*}
	and thus from \eqref{eq1},
	\begin{align*}
	    0&\neq H^T \tilde{F_k}^T \Big(\prod_{l=k+1}^{L-2}\Sigma_{i_l}^l\Big) (\Delta_{L-1})_{:p} \\
	  \Rightarrow  0&\neq H^T \tilde{F_k}^T \Big(\prod_{l=k+1}^{L-2}\Sigma_{i_l}^l\Big) \Delta_{L-1} 
	\end{align*}
	Compared to \eqref{eq:induction}, we have reduced the product from $\prod_{l=k+1}^{L-1}$ to $\prod_{l=k+1}^{L-2}$,
	By induction, one can easily show that 
	\begin{align*}
	    0&\neq H^T \tilde{F_k}^T \Sigma_{i_{k+1}}^{k+1} \Delta_{k+2} 
	\end{align*}
	and hence $0\neq H^T \tilde{F_k}^T \Delta_{k+1}$, 
  which implies $0\neq \tilde{F_k}^T \Delta_{k+1}=[(\nabla_{W_{k+1}})^T \Phi,\nabla_{b_{k+1}}\Phi]^T$.
  However, this is a contradiction to the fact that we assumed that $(W_l,b_l)_{l=1}^L$ is a critical point.
  Thus it follows that it has to hold $\Delta_L=0$.
	As $\Delta_L=0$ it holds \eqref{eq:loss-sep} which implies
	\begin{align*}
	    \begin{cases}
		f_{Lj}(x_i)\geq y_{ij} & x_i\in C_j \\ 
		f_{Lj}(x_i)\leq y_{ij} & x_i\notin C_j \\ 
	    \end{cases} \quad\forall\,i\in[N],j\in[m].
	\end{align*} 
	This in turn implies $\Phi\Big((W_l,b_l)_{l=1}^L\Big)=0.$ 
	Thus the critical point $(W_l,b_l)_{l=1}^L$ is a global minimum.
	
	\item This can be seen as a special case of the first statement. 
	In particular, assume one has a zero-layer which coincides with the training inputs,
	namely $F_0=X$, then the result follows immediately.
    \end{enumerate}
\end{proof}
\fi
Note that the second statement of Theorem \ref{theo:main2} can be considered as a special case of the first statement.
In the case where $L=2$ and training inputs are linearly separable,
the second statement of our Theorem \ref{theo:main2} recovers the similar result of \cite{Gori92, Frasconi97}
for one-hidden layer networks.

Even though the assumptions of Theorem \ref{theo:independent_inputs} and Theorem \ref{theo:main2} are different
in terms of class of activation and loss functions, their results are related.
In fact, it is well known that if a set of vectors is linearly independent then they are linearly separable, see e.g. p.340 \cite{Barber12}.
Thus Theorem \ref{theo:main2} can be seen as a direct generalization of Theorem \ref{theo:independent_inputs}. 
The caveat, which is also the main difference to Theorem \ref{theo:main}, is that
Theorem \ref{theo:main2} makes only statements for all the critical points 
for which the problem has become separable at some layer,
whereas there is no such condition in Theorem \ref{theo:main}. 
However, we still think that the result is of practical relevance, as one can expect for a sufficiently
large network that stochastic gradient descent will lead to a network structure where the data becomes separable at a particular layer. When this happens all
the associated critical points are globally optimal. 
It is an interesting question for further research if one can show directly under some architecture condition
that the network outputs become linearly separable at some layer
for any local minimum and thus every local minimum is a global minimum. 

\section{Discussion}
Our results show that the loss surface becomes well-behaved when there is a wide layer in the network.
Implicitly, such a wide layer is often present in convolutional neural networks used in computer vision.
It is thus an interesting future research question how and if our result can be generalized to neural networks
with sparse connectivity. We think that the results presented in this paper are a significant
addition to the recent understanding why deep learning works so efficiently. In particular, since in this paper we are 
directly working with the neural networks used in practice without any modifications or simplifications.

\section*{Acknowledgment}
The authors acknowledge support by the ERC starting grant NOLEPRO 307793.

\bibliography{regul}
\bibliographystyle{icml2017}

\end{document}